\documentclass[10pt, twoside]{article}
\usepackage[margin=1in]{geometry}

\usepackage[utf8]{inputenc} 
\usepackage[T1]{fontenc}    
\usepackage{url}            
\usepackage{booktabs}       
\usepackage{amsfonts}       
\usepackage{nicefrac}       
\usepackage{microtype}      

\usepackage{subfig}

\usepackage{amsfonts,amsmath,amssymb,amsthm,bm,bbm}

\usepackage{algorithm,algpseudocode}
\usepackage{wrapfig}
\usepackage{dsfont,array}

\usepackage{newclude}

\usepackage{appendix}

\usepackage[colorlinks,
linkcolor=blue,
citecolor=blue,
urlcolor=magenta]{hyperref}

\usepackage{multirow}
\usepackage{multicol}
\usepackage{xspace}
\usepackage{mathtools}
\usepackage{placeins}
\usepackage{authblk}

\newtheorem{definition}{Definition}
\newtheorem{assumption}{Assumption}
\newtheorem{theorem}{Theorem}
\newtheorem{lemma}{Lemma}

\newcommand{\indic}{1} \newcommand{\xyz}{u}
\newcommand{\cG}{\mathcal{G}} \newcommand{\ignore}[1]{}

\newcommand\numberthis{\addtocounter{equation}{1}\tag{\theequation}}

\DeclarePairedDelimiterX\Basics[1](){ #1}

\newcommand{\norm}[1]{\left\lVert#1\right\rVert}
\usepackage{centernot}

\newcounter{const-no}

\newcommand\independent{\protect\mathpalette{\protect\independenT}{\perp}}
\def\independenT#1#2{\mathrel{\rlap{$#1#2$}\mkern2mu{#1#2}}}

\def\EE{{\mathbb{E}}}\def\PP{{\mathbb{P}}}

\DeclareMathOperator*{\argmin}{\arg\!\min}

\makeatletter
\newcommand{\setword}[2]{%
	\phantomsection
	#1\def\@currentlabel{\unexpanded{#1}}\label{#2}%
}
\makeatother

\newcommand\blfootnote[1]{%
	\begingroup
	\renewcommand\thefootnote{}\footnote{#1}%
	\addtocounter{footnote}{-1}%
	\endgroup
}

\date{\today}
\title{Model-Powered Conditional Independence Test}

\author[1,*]{Rajat Sen}
\author[2,*]{Ananda Theertha Suresh}
\author[3,*]{Karthikeyan Shanmugam}
\author[1]{Alexandros G. Dimakis}
\author[1]{Sanjay Shakkottai}
\affil[1]{The University of Texas at Austin}
\affil[2]{Google, New York}
\affil[3]{IBM Research, Thomas J. Watson Center}

\usepackage[parfill]{parskip}
\begin{document}
	\maketitle 
\begin{abstract}
  We consider the problem of non-parametric Conditional Independence
  testing (CI testing) for continuous random variables. Given
  i.i.d samples from the joint distribution $f(x,y,z)$ of continuous
  random vectors $X,Y$ and $Z,$ we determine whether
  $X \independent Y \vert Z$. We approach this by converting the conditional independence test
  into a classification problem.  This allows us to harness very
  powerful classifiers like gradient-boosted trees and deep neural
  networks.  These models can handle complex probability distributions
  and allow us to perform significantly better compared to the prior
  state of the art, for high-dimensional CI testing. The main technical challenge in the classification problem is the
  need for samples from the conditional product distribution
  $f^{CI}(x,y,z) = f(x|z)f(y|z)f(z)$ -- the joint distribution if and
  only if $X \independent Y \vert Z.$ -- when given access only to
  i.i.d.  samples from the true joint distribution $f(x,y,z)$.  To
  tackle this problem we propose a novel nearest neighbor bootstrap
  procedure and theoretically show that our generated samples are
  indeed close to $f^{CI}$ in terms of total variational distance.
  We then develop theoretical results regarding the generalization
  bounds for classification for our problem, which translate into
  error bounds for CI testing. We provide a novel analysis
  of Rademacher type classification bounds in the presence of
  non-i.i.d \textit{near-independent} samples. We empirically validate
  the performance of our algorithm on simulated and real datasets and
  show performance gains over previous methods.
  \blfootnote{* Equal Contribution}
\end{abstract}

\section{Introduction}
%
%
%
%
%
%
%
%
%
%
%

Testing datasets for Conditional Independence (CI) have significant
applications in several statistical/learning problems; among
others, examples include discovering/testing for edges in Bayesian
networks~\cite{koller2009probabilistic,
  spirtes2000causation,cheng1998learning,de2000new}, causal
inference~\cite{pearl2009causality,
  kalisch2007estimating,tsamardinos2006max, brenner2013sparsityboost}
and feature selection through Markov
Blankets~\cite{koller1996toward,xing2001feature}. Given a 
triplet of random variables/vectors $(X, Y, Z)$, we say that $X$ is
conditionally independent of $Y$ given $Z$ (denoted by
$X \independent Y \vert Z$), if the joint distribution
$f_{X,Y,Z}(x,y,z)$ factorizes as
$f_{X,Y,Z}(x,y,z) = f_{X|Z}(x|z)f_{Y|Z}(y|z)f_{Z}(z)$. The problem of
\textit{Conditional Independence Testing} (CI Testing) can be defined
as follows: Given $n$ i.i.d samples from $f_{X,Y,Z}(x,y,z)$,
distinguish between the two hypothesis $\mathcal{H}_0:$
$X \independent Y \vert Z$ and $\mathcal{H}_1:$
$X \centernot{\independent} Y \vert Z$.



In this paper we propose a data-driven \textit{Model-Powered} CI
test. The central idea in a model-driven approach is to convert a
statistical testing or estimation problem into a pipeline that
utilizes the power of supervised learning models like classifiers and
regressors; such pipelines can then leverage recent advances in
classification/regression in high-dimensional settings. In this paper,
we take such a model-powered approach (illustrated in
Fig.~\ref{fig:illustrate}), which reduces the problem of CI testing to
Binary Classification. Specifically, the key steps of our procedure
are as follows:

\begin{figure}[h] \centering
  \includegraphics[width=8cm,height=4.5cm]{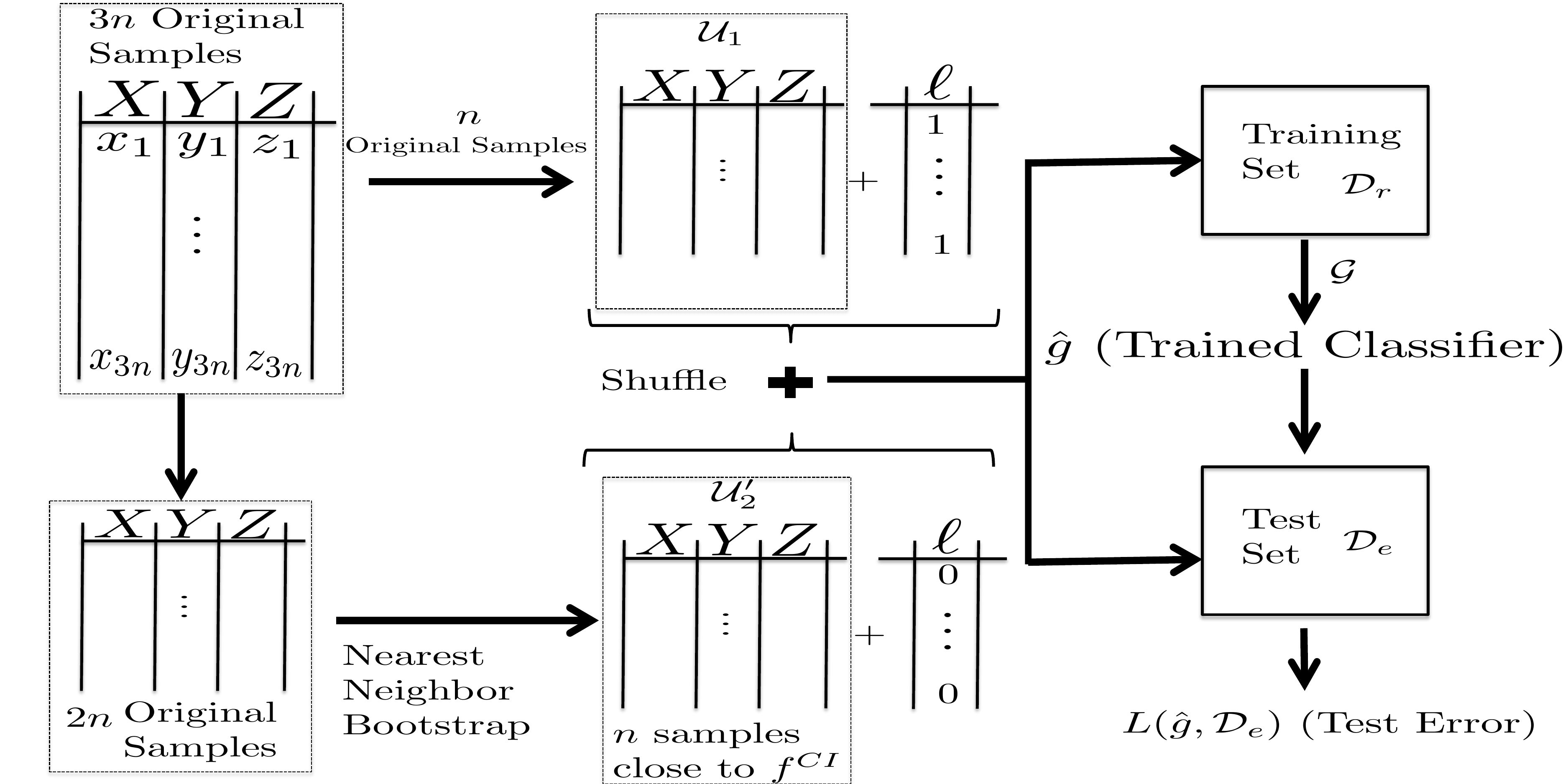}
  \caption{Illustration of our methodology. A part of the original
    samples are kept aside in $\mathcal{U}_1$. The rest of the samples
    are used in our nearest neighbor boot-strap to generate a data-set
    $\mathcal{U}_2'$ which is close to $f^{CI}$ in distribution. The
    samples are labeled as shown and a classifier is trained on a
    training set. The test error is measured on a test set
    there-after. If the test-error is close to $0.5$, then
    $\mathcal{H}_0$ is not rejected, however if the test error is low
    then $\mathcal{H}_0$ is rejected.}
	\label{fig:illustrate}
\end{figure}

 
\setword{$(i)$}{step1} Suppose we are provided $3n$ i.i.d samples from
$f_{X,Y,Z}(x,y,z)$. We keep aside $n$ of these original samples in a
set $\mathcal{U}_1$ (refer to Fig.~\ref{fig:illustrate}).  The
remaining $2n$ of the original samples are processed through our first
module, the \textit{nearest-neighbor bootstrap}
(Algorithm~\ref{alg:NNB} in our paper), which produces $n$ simulated
samples stored in $\mathcal{U}_2'$. In Section~\ref{sec:results}, we
show that these generated samples in $\mathcal{U}_2'$ are in fact
close in total variational distance (defined in
Section~\ref{sec:results}) to the conditionally independent
distribution
$f^{CI}(x,y,z) \triangleq f_{X|Z}(x|z)f_{Y|Z}(y|z)f_Z(z)$.  (Note that
only under $\mathcal{H}_0$ does the equality
$f^{CI}(.) = f_{X,Y,Z}(.)$ hold; our method generates samples close to
$f^{CI}(x,y,z)$ under {\em both} hypotheses).

\setword{$(ii)$}{step2} Subsequently, the original samples kept aside
in $\mathcal{U}_1$ are labeled $1$ while the new samples simulated
from the nearest-neighbor bootstrap (in $\mathcal{U}_2'$) are labeled
$0$. The labeled samples ($\mathcal{U}_1$ with label $1$ and
$\mathcal{U}_2'$ labeled $0$) are aggregated into a data-set
$\mathcal{D}$. This set $\mathcal{D}$ is then broken into training and test
sets $\mathcal{D}_r$ and $\mathcal{D}_e$ each containing $n$ samples
each.

\setword{$(iii)$}{step3} Given the labeled training data-set (from
step~\ref{step2}), we train powerful classifiers such as gradient boosted
trees~\cite{chen2016xgboost} or deep neural
networks~\cite{krizhevsky2012imagenet} which attempt to learn the
classes of the samples. If the trained classifier has good accuracy
over the test set, then intuitively it means that the joint
distribution $f_{X,Y,Z}(.)$ is distinguishable from $f^{CI}$ (note
that the generated samples labeled $0$ are close in distribution to
$f^{CI}$). Therefore, we reject $\mathcal{H}_0$. On the other hand, if
the classifier has accuracy close to random guessing, then
$f_{X,Y,Z}(.)$ is in fact close to $f^{CI}$, and we fail to reject
$\mathcal{H}_0$.

For independence testing (i.e whether $X \independent Y$), classifiers
were recently used in~\cite{lopez2016revisiting}. Their key
observation was that given i.i.d samples $(X,Y)$ from $f_{X,Y}(x,y)$, if the
$Y$ coordinates are randomly permuted then the resulting samples
exactly emulate the distribution $f_X(x)f_Y(y)$. Thus the problem can
be converted to a two sample test between a subset of the original
samples and the other subset which is permuted - Binary classifiers
were then harnessed for this two-sample testing; for details see
\cite{lopez2016revisiting}. However, in the case of CI testing we need
to emulate samples from $f^{CI}$. This is harder because the
permutation of the samples needs to be $Z$ dependent (which can be
high-dimensional). One of our key technical contributions is in
proving that our nearest-neighbor bootstrap in step~\ref{step1}
achieves this task.

The advantage of this modular approach is that we can harness the
power of classifiers (in step~\ref{step3} above), which have good
accuracies in high-dimensions. Thus, any improvements in the field of
binary classification imply an advancement in our CI test. Moreover,
there is added flexibility in choosing the best classifier based on
domain knowledge about the data-generation process. Finally, our
bootstrap is also efficient owing to fast algorithms for identifying
nearest-neighbors~\cite{ramasubramanian1992fast}.

\subsection{Main Contributions}
\label{contri}
$(i)$ {\bf (Classification based CI testing)} We reduce the problem of
CI testing to Binary Classification as detailed in
steps~\ref{step1}-\ref{step3} above and in
Fig.~\ref{fig:illustrate}. We simulate samples that are close to
$f^{CI}$ through a novel nearest-neighbor bootstrap
(Algorithm~\ref{alg:NNB}) given access to i.i.d samples from the joint
distribution. The problem of CI testing then reduces to a two-sample
test between the original samples in $\mathcal{U}_1$ and
$\mathcal{U}_2'$, which can be effectively done by binary classifiers.

$(ii)$ {\bf (Guarantees on Bootstrapped Samples)} As mentioned in
steps~\ref{step1}-\ref{step3}, if the samples generated by the
bootstrap (in $\mathcal{U}_2'$) are close to $f^{CI}$, then the CI
testing problem reduces to testing whether the data-sets
$\mathcal{U}_1$ and $\mathcal{U}_2'$ are distinguishable from each
other.  We theoretically justify that this is indeed true. Let
$\phi_{X,Y,Z}(x,y,z)$ denote the distribution of a sample produced by
Algorithm~\ref{alg:NNB}, when it is supplied with $2n$ i.i.d samples
from $f_{X,Y,Z}(.)$. In Theorem~\ref{thm:dtv}, we prove that
$d_{TV}(\phi, f^{CI}) = O(1/n^{1/d_z})$ under appropriate smoothness
assumptions. Here $d_z$ is the dimension of $Z$ and $d_{TV}$ denotes
total variational distance (Def.~\ref{dtv}).


$(iii)$ {\bf (Generalization Bounds for Classification under
  \textit{near-independence})} The samples generated from the
nearest-neighbor bootstrap do not remain i.i.d but they are
\textit{close} to i.i.d. We quantify this property and go on to show
generalization risk bounds for the classifier. Let us denote the class
of function encoded by the classifier as $\mathcal{G}$. Let $\hat{R}$
denote the probability of error of the optimal classifier
$\hat{g} \in \mathcal{G}$ trained on the training set
(Fig.~\ref{fig:illustrate}). We prove that under appropriate
assumptions, we have
\begin{align*} &r_{0} - \mathcal{O}(1/n^{1/d_z}) \leq \hat{R} \leq r_{0} +
  \mathcal{O}(1/n^{1/d_z}) + \mathcal{O}\left(\sqrt{V} \left( n^{-1/3}
      + \sqrt{2^{d_z}/n} \right)\right) \end{align*} with high
probability, upto log factors. Here $r_0 = 0.5(1 - d_{TV}(f,f^{CI}))$,
$V$ is the VC dimension~\cite{vapnik2015uniform} of the class
$\mathcal{G}$.  Thus when $f$ is equivalent to $f^{CI}$
($\mathcal{H}_0$ holds) then the error rate of the classifier is close
to $0.5$. But when $\mathcal{H}_1$ holds the loss is much lower.  We
provide a novel analysis of Rademacher complexity
bounds~\cite{boucheron2005theory} under near-independence which is of
independent interest.

$(iv)$ {\bf (Empirical Evaluation)} We perform extensive numerical
experiments where our algorithm outperforms the state of the
art~\cite{zhang2012kernel,strobl2017approximate}. We also apply our
algorithm for analyzing CI relations in the protein signaling network
data from the flow cytometry data-set~\cite{sachs2005causal}. In
practice we observe that the performance with respect to dimension of
$Z$ scales much better than expected from our worst case theoretical
analysis. This is because powerful binary classifiers perform well in
high-dimensions.

\subsection{Related Work}
In this paper we address the problem of non-parametric CI testing when
the underlying random variables are continuous. The literature on
non-parametric CI testing is vast. We will review some of the recent
work in this field that is most relevant to our paper.  

Most of the recent work in CI testing are kernel
based~\cite{strobl2017approximate,zhang2012kernel,doran2014permutation}.
Many of these works build on the study
in~\cite{fukumizu2004dimensionality}, where non-parametric CI
relations are characterized using covariance operators for Reproducing
Kernel Hilbert Spaces (RKHS)
\cite{fukumizu2004dimensionality}. KCIT~\cite{zhang2012kernel} uses
the partial association of regression functions relating $X$, $Y$ ,
and $Z$.  RCIT~\cite{strobl2017approximate} is an approximate version
of KCIT that attempts to improve running times when the number of
samples are large. KCIPT~\cite{doran2014permutation} is perhaps most
relevant to our work. In~\cite{doran2014permutation}, a specific
permutation of the samples is used to simulate data from $f^{CI}$. An
expensive linear program needs to be solved in order to calculate the
permutation. On the other hand, we use a simple nearest-neighbor
bootstrap and further we provide theoretical guarantees about the
closeness of the samples to $f^{CI}$ in terms of total variational
distance. Finally the two-sample test in~\cite{doran2014permutation}
is based on a kernel method~\cite{borgwardt2006integrating}, while we
use binary classifiers for the same purpose. There has also been
recent work on entropy estimation~\cite{gao2016demystifying} using
nearest neighbor techniques (used for density estimation); this can
subsequently be used for CI testing by estimating the conditional
mutual information $\mathrm{I}(X;Y \vert Z)$.
 
Binary classification has been recently used for two-sample testing,
in particular for independence testing~\cite{lopez2016revisiting}. Our
analysis of generalization guarantees of classification are aimed at
recovering guarantees similar to~\cite{boucheron2005theory}, but in a
non-i.i.d setting. In this regard (non-i.i.d generalization
guarantees), there has been recent work in proving Rademacher
complexity bounds for $\beta$-mixing stationary
processes~\cite{mohri2009rademacher}. This work also falls in the category of machine learning reductions, where the general philosophy is to reduce 
various machine learning settings like multi-class regression~\cite{beygelzimer2009conditional}, ranking~\cite{balcan2007robust}, reinforcement learning~\cite{langford2003reducing}, structured prediction~\cite{daume2009search} to that of binary classification.


\section{Problem Setting and Algorithms}
\label{sec:Algorithm}
In this section we describe the algorithmic details of our CI testing procedure. We first formally define our problem. Then we describe our bootstrap algorithm for generating the data-set that mimics samples from $f^{CI}$. We give a detailed pseudo-code for our CI testing process which reduces the problem to that of binary classification. Finally, we suggest further improvements to our algorithm. 

{\bf Problem Setting: }The problem setting is that of non-parametric \textit{Conditional Independence (CI)} testing given i.i.d samples from the joint distributions of random variables/vectors~\cite{zhang2012kernel,doran2014permutation,strobl2017approximate}. We are given $3n$ i.i.d samples from a continuous joint distribution $f_{X,Y,Z}(x,y,z)$ where $x \in \mathbb{R}^{d_x}, y \in \mathbb{R}^{d_y}$ and $z \in \mathbb{R}^{d_z}$. The goal is to test whether $X \independent Y \vert Z$ i.e whether $f_{X,Y,Z}(x,y,z)$ factorizes as,
$f_{X,Y,Z}(x,y,z) = f_{X \vert Z}(x \vert z)f_{Y \vert Z}(y \vert z)f_{Z}(z) \label{eq:CI} \triangleq f^{CI}(x,y,z)$

This is essentially a hypothesis testing problem where: $\mathcal{H}_0 \text{ : } X \independent Y \vert Z$ and 
$\mathcal{H}_1 \text{ : } X \centernot{\independent} Y \vert Z$.

{\bf Note: } For notational convenience, we will drop the subscripts when the context is evident. For instance we may use $f(x \vert z)$ in place of $f_{X \vert Z}(x \vert z)$. 

{\bf Nearest-Neighbor Bootstrap: }
Algorithm~\ref{alg:NNB} is a procedure to generate a data-set $\mathcal{U'}$ consisting of $n$ samples given a data-set $\mathcal{U}$ of $2n$ i.i.d samples from the distribution $f_{X,Y,Z}(x,y,z)$. The data-set $\mathcal{U}$ is broken into two equally sized partitions $\mathcal{U}_1$ and $\mathcal{U}_2$. Then for each sample in $\mathcal{U}_1$, we find the nearest neighbor in $\mathcal{U}_2$ in terms of the $Z$ coordinates. The $Y$-coordinates of the sample from $\mathcal{U}_1$ are exchanged with the $Y$-coordinates of its nearest neighbor (in $\mathcal{U}_2$); the modified sample is added to $\mathcal{U'}$. 
\begin{algorithm}
	\caption{DataGen - Given data-set $\mathcal{U} = \mathcal{U}_1 \cup \mathcal{U}_2$ of $2n$ i.i.d samples from $f(x,y,z)$ ($\vert \mathcal{U}_1 \vert = \vert \mathcal{U}_2 \vert = n$ ), returns a new data-set $\mathcal{U}'$ having $n$ samples.}
	\begin{algorithmic}[1]
		\Function{DataGen}{$\mathcal{U}_1,\mathcal{U}_2,2n$} 
		\State $\mathcal{U}' = \emptyset$
		\For {$u$ in $\mathcal{U}_1$}  	
		\State \parbox[t]{\dimexpr\linewidth-\algorithmicindent}{Let $v = (x',y',z') \in \mathcal{U}_2$ be the sample such that $z'$ is the $1$-Nearest Neighbor (1-NN) \\ of $z$ (in $\ell_2$ norm) in the whole data-set ${\cal U}_2$, where $u = (x,y,z)$}
		\State Let $u' = (x,y',z)$ and $\mathcal{U}' = \mathcal{U}'\cup \{u' \}.$
		\EndFor
		\EndFunction
	\end{algorithmic}
	\label{alg:NNB}
\end{algorithm}

One of our main results is that the samples in $\mathcal{U}'$, generated in Algorithm~\ref{alg:NNB} mimic samples coming from the distribution $f^{CI}$. Suppose $u = (x,y,z) \in \mathcal{U}_1$ be a sample such that $f_Z(z)$ is not too small. In this case $z'$ (the 1-NN sample from $\mathcal{U}_2$) will not be far from $z$. Therefore given a fixed $z$, under appropriate smoothness assumptions, $y'$ will be close to an independent sample coming from $f_{Y|Z}(y|z') \sim f_{Y|Z}(y|z)$. On the other hand if $f_{Z}(z)$ is small, then $z$ is a rare occurrence and will not contribute adversely. 

{\bf CI Testing Algorithm: }
Now we introduce our CI testing algorithm, which uses Algorithm~\ref{alg:NNB} along with binary classifiers. The psuedo-code is in Algorithm~\ref{alg:ccitv1} (Classifier CI Test -CCIT).  
\begin{algorithm}
	\caption{CCITv1 - Given data-set $\mathcal{U}$ of $3n$ i.i.d samples from $f(x,y,z)$, returns if $X \independent Y \vert Z$.}
	\begin{algorithmic}[1]
		\Function{CCIT}{$\mathcal{U},3n,\tau, \mathcal{G}$}
		\State Partition $\mathcal{U}$ into three disjoint partitions $\mathcal{U}_1$, $\mathcal{U}_2$ and $\mathcal{U}_3$ of size $n$ each, randomly.
		\State Let $\mathcal{U}_2'$ = DataGen($\mathcal{U}_2 , \mathcal{U}_3,2n$) (Algorithm~\ref{alg:NNB}). Note that $\lvert \mathcal{U}_2' \rvert = n$. 
		\State Create Labeled data-set $\mathcal{D} := \{ (u,\ell = 1)\}_{u \in \mathcal{U}_1} \cup \{ (u',\ell' = 0)\}_{u' \in \mathcal{U}_2'}$  
		\State Divide data-set $\mathcal{D}$ into train and test set $\mathcal{D}_{r}$ and $\mathcal{D}_{e}$ respectively. Note that $\vert\mathcal{D}_{r} \vert = \vert\mathcal{D}_{e} \vert = n$. 
		\State \parbox[t]{\dimexpr\linewidth-\algorithmicindent}{Let $\hat{g} = \argmin_{g \in \mathcal{G}} \hat{L}(g,\mathcal{D}_{r}) := \frac{1}{\lvert \mathcal{D}_r\rvert}\sum _{(u,\ell) \in \mathcal{D}_r}\mathds{1} \{g(u) \neq l \}$. This is Empirical Risk Minimization for training the classifier (finding the best function in the class $\mathcal{G}$).} 
		\State If $\hat{L}(\hat{g},\mathcal{D}_{e}) > 0.5 - \tau$, then conclude $X \independent Y \vert Z$, otherwise, conclude $X \centernot{\independent} Y \vert Z$.
		\EndFunction
	\end{algorithmic}
	\label{alg:ccitv1}
\end{algorithm}

In Algorithm~\ref{alg:ccitv1}, the original samples in $\mathcal{U}_1$ and the nearest-neighbor bootstrapped samples in $\mathcal{U}_2'$ should be almost indistinguishable if $\mathcal{H}_0$ holds. However, if $\mathcal{H}_1$ holds, then the classifier trained in Line 6 should be able to easily distinguish between the samples corresponding to different labels. In Line 6, $\mathcal{G}$ denotes the space of functions over which risk minimization is performed in the classifier.

We will show (in Theorem~\ref{thm:dtv}) that the variational distance between the distribution of one of the  samples in $\mathcal{U}_2'$ and $f^{CI}(x,y,z)$ is very small for large $n$. However, the samples in $\mathcal{U}_2'$ are not exactly i.i.d but \textit{close} to i.i.d. Therefore, in practice for finite $n$, there is a small bias $b>0$ i.e. $\hat{L}(\hat{g},\mathcal{D}_{e}) \sim 0.5-b$, even when $\mathcal{H}_0$ holds.  The threshold $\tau$ needs to be greater than $b$ in order for Algorithm~\ref{alg:ccitv1} to function. In the next section, we present an algorithm where this bias is corrected.

{\bf Algorithm with Bias Correction: }
We present an improved bias-corrected version of our algorithm as Algorithm~\ref{alg:ccitv2}. As mentioned in the previous section, in Algorithm~\ref{alg:ccitv1}, the optimal classifier may be able to achieve a loss slightly less that 0.5 in the case of finite $n$, even when $\mathcal{H}_{0}$ is true. However, the classifier is expected to distinguish between the two data-sets only based on the $Y, Z$ coordinates, as the joint distribution of  $X$ and $Z$ remains the same in the nearest-neighbor bootstrap. The key idea in Algorithm~\ref{alg:ccitv2} is to train a classifier only using the $Y$ and $Z$ coordinates, denoted by $\hat{g}'$. As before we also train another classier using all the coordinates, which is denoted by $\hat{g}$. The test loss of $\hat{g}'$ is expected to be roughly $0.5 - b$, where $b$ is the bias mentioned in the previous section. Therefore, we can just subtract this bias. Thus, when $\mathcal{H}_0$ is true $\hat{L}(\hat{g}',\mathcal{D}_{e}') - \hat{L}(\hat{g},\mathcal{D}_{e})$ will be close to $0$. However, when $\mathcal{H}_1$ holds, then $\hat{L}(\hat{g},\mathcal{D}_{e})$ will be much lower, as the classifier $\hat{g}$ has been trained leveraging the information encoded in all the coordinates. 

\begin{algorithm}
	\caption{CCITv2 - Given data-set $\mathcal{U}$ of $3n$ i.i.d samples, returns whether $X \independent Y \vert Z$.}
	\begin{algorithmic}[1]
		\Function{CCIT}{$\mathcal{U},3n,\tau, \mathcal{G}$}
		\State  Perform Steps 1-5 as in Algorithm~\ref{alg:ccitv1}. 
		\State 	\parbox[t]{\dimexpr\linewidth-\algorithmicindent}{Let $\mathcal{D}_{r}' = \{ ((y,z),\ell)\}_{(u = (x,y,z),\ell)\in \mathcal{D}_{r}}$. Similarly, let $\mathcal{D}_{e}' = \{ ((y,z),\ell)\}_{(u = (x,y,z),\ell)\in \mathcal{D}_{e}}$. These are the training and test sets without the $X$-coordinates.}
		\State  Let $\hat{g} = \argmin_{g \in \mathcal{G}} \hat{L}(g,\mathcal{D}_{r}) := \frac{1}{\lvert \mathcal{D}_r\rvert}\sum _{(u,\ell) \in \mathcal{D}_r}\mathds{1} \{g(u) \neq l \}$. Compute test loss: $\hat{L}(\hat{g},\mathcal{D}_{e}).$ 
		\State Let $\hat{g}' = \argmin_{g \in \mathcal{G}} \hat{L}(g,\mathcal{D}'_{r}) := \frac{1}{\lvert \mathcal{D}_{r}'\rvert}\sum _{(u,\ell) \in \mathcal{D}_r'}\mathds{1} \{g(u) \neq l \}$. Compute test loss: $\hat{L}(\hat{g}',\mathcal{D}'_{e}).$
		\State If $\hat{L}(\hat{g},\mathcal{D}_{e}) < \hat{L}(\hat{g}',\mathcal{D}'_{e}) - \tau$, then conclude $X \centernot{\independent} Y \vert Z$, otherwise, conclude $X \independent Y \vert Z$.
		\EndFunction
	\end{algorithmic}
	\label{alg:ccitv2}
\end{algorithm}

\section{Theoretical Results}
\label{sec:results}
In this section, we provide our main theoretical results. We first show that the distribution of any one of the samples generated in Algorithm~\ref{alg:NNB} closely resemble that of a sample coming from $f^{CI}$. This result holds for a broad class of distributions $f_{X,Y,Z}(x,y,z)$ which satisfy some smoothness assumptions. However, the samples generated by Algorithm~\ref{alg:NNB} ($\mathcal{U}_2$ in the algorithm) are not exactly i.i.d but \textit{close} to i.i.d. We quantify this and go on to show that empirical risk minimization over a class of classifier functions generalizes well using these samples. 
Before, we formally state our results we provide some useful definitions. 

\begin{definition} \label{dtv} The \textbf{total variational distance} between two continuous probability distributions $f(.)$ and $g(.)$ defined over a domain $\mathcal{X}$ is, $d_{TV}(f,g) = \sup_{p \in \mathcal{B}} \lvert\EE_{f}[p(X)] - \EE_{g}[p(X)] \rvert$ where $\mathcal{B}$ is the set of all measurable functions from $\mathcal{X} \rightarrow [0,1]$. Here, $\EE_{f}[.]$ denotes expectation under distribution $f$. 

\end{definition}



 We first prove that the distribution of any one of the samples generated in Algorithm~\ref{alg:NNB} is close to $f^{CI}$ in terms of total variational distance. We make the following assumptions on the joint distribution of the original samples i.e. $f_{X,Y,Z}(x,y,z)$:

{\bf Smoothness assumption on $f(y|z)$: } We assume a smoothness condition on $f(y|z)$, that is a generalization of boundedness of the max. eigenvalue of Fisher Information matrix of $y$ w.r.t $z$.
\begin{assumption}
	\label{assump1}
	For $z \in \mathbb{R}^{d_z}$, $a$ such that $\lVert a - z \rVert_2 \leq \epsilon_1$, the generalized curvature matrix $\mathbf{I}_{a}(z)$ is,
	\begin{equation}
	\mathbf{I}_{a}(z)_{ij} = \left(\frac{\partial^2}{\partial z'_i \partial z'_j} \int \log \frac{f(y|z)}{f(y|z')}  f(y|z) dy \right) \Bigg \vert _{z' = a} = \EE \left[ -\frac{\delta^2 \log f(y|z')}{\delta z'_i \delta z'_j} \Big \vert _{z' = a} \Bigg \vert Z = z\right]
	\end{equation}
	We require that for all $z \in \mathbb{R}^{d_z}$ and all $a$ such that $\lVert a - z \rVert_2 \leq \epsilon_1$, $\lambda_{max} \left( \mathbf{I}_{a}(z) \right) \leq \beta$. Analogous assumptions have been made on the Hessian of the density in the context of entropy estimation~\cite{gao2016breaking}. 
\end{assumption}

{\bf Smoothness assumptions on $f(z)$: }
We assume some smoothness properties of the probability density function $f(z)$. The smoothness assumptions (in Assumption~\ref{assumption1}) is a subset of the assumptions made in \cite{gao2016demystifying} (Assumption 1, Page 5) for entropy estimation.

\begin{definition}
	\label{def:G}
	For any $\delta >0$, we define  $G(\delta) = \PP \left( f(Z) \leq \delta \right)$. This is the probability mass of the distribution of $Z$ in the areas where the p.d.f is less than $\delta$. 
\end{definition}

\begin{definition}
(Hessian Matrix) Let $H_f(z)$ denote the Hessian Matrix of the p.d.f $f(z)$ with respect to $z$ i.e $H_f(z)_{ij} = \partial^2f(z)/\partial z_i \partial z_j$, provided it is twice continuously differentiable at $z$. 
\end{definition}

\begin{assumption}\label{assumption1}
	The probability density function $f(z)$ satisfies the following:
	
	(1) $f(z)$ is twice continuously differentiable and the Hessian matrix $H_f$ satisfies $\lVert H_f(z) \rVert_2 \leq c_{d_z}$ almost everywhere, where $c_{d_z}$ is only dependent on the dimension. 
	
	(2) $\int f(z)^{1-1/d} dz \leq c_3 ,~ \forall d \geq 2 $ where $c_3$ is a constant. 
\end{assumption}

\begin{theorem}
	\label{thm:dtv}
	Let $(X,Y',Z)$ denote a sample in $\mathcal{U}_2'$ produced by Algorithm~\ref{alg:NNB} by modifying the original sample $(X,Y,Z)$ in $\mathcal{U}_1$, when supplied with $2n$ i.i.d samples from the original joint distribution $f_{X,Y,Z}(x,y,z)$. Let $\phi_{X,Y,Z}(x,y,z)$ be the distribution of $(X,Y',Z)$. Under smoothness assumptions (\ref{assump1}) and (\ref{assumption1}), for any $\epsilon < \epsilon_1$, $n$ large enough, we have:
	\begin{align*}
	&d_{TV}(\phi,f^{CI}) \leq  b(n) \\
	&\triangleq  \frac{1}{2}\sqrt{\frac{\beta}{4} \frac{c_3 * 2^{1/d_z} \Gamma(1/d_z) }{(n \gamma_{d_z})^{1/d_z} d_z}  + \frac{\beta \epsilon G  \left(2c_{d_z} \epsilon^2 \right) }{4} }+  \exp \left( -\frac{1}{2}n \gamma_{d_z} c_{d_z} \epsilon^{d_z+2} \right) + G \left( 2 c_{d_z} \epsilon^2 \right).
	\end{align*}
	Here, $\gamma_d$ is the volume of the unit radius $\ell_2$ ball in $\mathbb{R}^d$.
\end{theorem}

Theorem~\ref{thm:dtv} characterizes the variational distance of the distribution of a sample generated in Algorithm~\ref{alg:NNB} with that of the conditionally independent distribution $f^{CI}$. We defer the proof of Theorem~\ref{thm:dtv} to Appendix~\ref{proof:boot}. Now, our goal is to characterize the misclassification error of the trained classifier in Algorithm~\ref{alg:ccitv1} under both $\mathcal{H}_0$ and $\mathcal{H}_1$. Consider the distribution of the samples in the data-set $\mathcal{D}_r$ used for classification in Algorithm~\ref{alg:ccitv1}. Let $q(x,y,z \vert \ell = 1)$ be the marginal distribution of each sample with label $1$. Similarly, let $q(x,y,z \vert \ell = 0)$ denote the marginal distribution of the label $0$ samples. Note that under our construction, 
\begin{align*}
q(x,y,z \vert \ell = 1) &= f_{X,Y,Z}(x,y,z) = \left\{
\begin{array}{ll}
f^{CI}(x,y,z)  & \mbox{if } \mathcal{H}_0 \mbox{ holds }\\
\neq  f^{CI}(x,y,z) & \mbox{if }  \mathcal{H}_1 \mbox{ holds }
\end{array}
\right. \\
q(x,y,z \vert \ell = 0) &= \phi_{X,Y,Z}(x,y,z)   \numberthis \label{eq:classification}
\end{align*} 
where $\phi_{X,Y,Z}(x,y,z)$ is as defined in Theorem~\ref{thm:dtv}. 

Note that even though the marginal of each sample with label $0$ is $\phi_{X,Y,Z}(x,y,z)$ (Equation~\eqref{eq:classification}), they are not exactly i.i.d owing to the nearest neighbor bootstrap. We will go on to show that they are actually \textit{close} to i.i.d and therefore classification risk minimization generalizes similar to the i.i.d results for classification~\cite{boucheron2005theory}. First, we review standard definitions and results from classification theory~\cite{boucheron2005theory}. 

{\bf Ideal Classification Setting:} We consider an \textit{ideal} classification scenario for CI testing and in the process define standard quantities in learning theory. Recall that $\cG$ is the set of classifiers under consideration. Let $\tilde{q}$
be our \emph{ideal} distribution for $q$ given by $\tilde{q}(x,y,z|\ell=1) = f_{X,Y,Z}(x,y,z)$, $\tilde{q}(x,y,z|\ell=0) = f^{CI}_{X,Y,Z}(x,y,z)$ and $\tilde{q}(\ell=1) = \tilde{q}(\ell=0) = 0.5$. In other words this is the ideal classification scenario for testing CI. 
Let $L(g(\xyz),\ell)$ be our \textbf{loss function} for a classifying function $g \in \mathcal{G}$, for a sample $u \triangleq (x,y,z)$ with true label $\ell$. In our algorithms the loss function is the $0-1$ loss, but our results hold for any bounded loss function s.t. $|L(g(u),\ell)| \leq |L|$. For a distribution $\tilde{q}$ and a classifier $g$ let $R_{\tilde{q}}(g) \triangleq \EE_{u,\ell \sim \tilde{q}} [L(g(\xyz),\ell)]$ be the \textbf{expected risk} of the function $g$. 
The \textbf{risk optimal classifier} $g^*_{\tilde{q}}$ under $\tilde{q}$ is given by $g^*_{\tilde{q}} \triangleq \arg \min_{g \in \cG} R_{\tilde{q}}(g)$. Similarly for a set of samples $S$ and a classifier $g$, let $R_S(g) \triangleq \frac{1}{|S|}\sum_{u,\ell \in S} L(g(\xyz),\ell)$ be the \textbf{empirical risk} on the set of samples. 
We define $g_S$ as the classifier that \textbf{minimizes the empirical loss} on the observed set of samples $S$ that is,  $g_S \triangleq \arg \min_{g \in \cG} R_{S}(g)$.

If the samples in $S$ are generated independently from $\tilde{q}$, then
standard results from the learning theory states that with probability
$\geq 1-\delta$,
\[
R_{\tilde{q}}(g_S) \leq R_{\tilde{q}}(g^*_{\tilde{q}}) + C
\sqrt{\frac{V}{n}} + \sqrt{\frac{2 \log (1/\delta)}{n}}, \numberthis \label{eq:stdcc}
\]
where $V$ is the VC dimension~\cite{vapnik2015uniform} of the classification model, $C$ is
an universal constant and $n = \vert S \vert$.

{\bf Guarantees under near-independent samples: } Our goal is to prove a result like~\eqref{eq:stdcc}, for the classification problem in Algorithm~\ref{alg:ccitv1}. However, in this case we do not have access to i.i.d samples because the samples in $\mathcal{U}_2'$ do not remain independent. We will see that they are close to independent in some sense. This brings us to one of our main results in Theorem~\ref{thm:mainclass}. 

\begin{theorem}
	\label{thm:mainclass}
	Assume that the joint distribution $f(x,y,z)$ satisfies the conditions in Theorem~\ref{thm:dtv}. Further assume that $f(z)$ has a bounded Lipschitz constant. Consider the classifier $\hat{g}$ in Algorithm~\ref{alg:ccitv1} trained on the set $\mathcal{D}_r$. Let $S = \mathcal{D}_r$. Then according to our definition $g_{S} = \hat{g}$. For $\epsilon>0$ we have:
	\begin{align*}
	(i) \text{  } &R_{{q}}(g_{S}) - R_{{q}} ({g}^*_{q}) \leq  \gamma_n \\ 
	& \triangleq C \lvert L \rvert \left( \left(\sqrt{V}
	+ \sqrt{\log \frac{1}{\delta}} \right)\left( \left(\frac{\log
		(n/\delta)}{n}\right)^{1/3} + \sqrt{\frac{4^{d_z}\log (n/\delta)  + o_n(1/\epsilon)}{n}}\right)
	+ G(\epsilon) \right),
	\end{align*}
	with probability at least $1 - 8\delta$. Here $V$ is the V.C. dimension of the classification function class, $G$ is as defined in Def.~\ref{def:G}, $C$ is an universal constant and $\lvert L \rvert$ is the bound on the absolute value of the loss. 
	
	$(ii)$ Suppose the loss is $L(g(u),\ell) = \mathds{1}_{g(u) \neq \ell}$ (s.t $|L| \leq 1$). Further suppose the class of classifying functions is such that $R_{q}(g^*_q) \leq r_{0} + \eta$. Here, 
$r_0 \triangleq 0.5(1 - d_{TV}(q(x,y,z \vert 1), q(x,y,z \vert 0)))$ is the risk of the Bayes optimal classifier when 
	$q(\ell = 1) = q(\ell=0)$. This is the best loss that any classifier can achieve for this classification problem~\cite{boucheron2005theory}. Under this setting, w.p at least $1 - 8\delta$ we have:
	\begin{align*}
	  \frac{1}{2} \left(1 - d_{TV}(f,f^{CI}) \right) - \frac{b(n)}{2}  \leq R_{{q}}(g_{S}) \leq  \frac{1}{2} \left(1 - d_{TV}(f,f^{CI}) \right) + \frac{b(n)}{2} + \eta + \gamma_n
	\end{align*}
	where $b(n)$ is as defined in Theorem~\ref{thm:dtv}.
\end{theorem}

We prove Theorem~\ref{thm:mainclass} as Theorem~\ref{thm:general} and Theorem~\ref{thm:riskdtv} in the appendix. In part $(i)$ of the theorem we prove that generalization bounds hold even when the samples are not exactly i.i.d. Intuitively, consider two sample inputs $u_i, u_j \in \mathcal{U}_1$, such
that corresponding $Z$ coordinates $z_i$ and $z_j$ are far away. Then we expect the
resulting samples $u'_i$ and $u'_j$ (in $\mathcal{U}_2'$) to be nearly-independent. By
carefully capturing this notion of spatial near-independence, we
prove generalization errors in Theorem~\ref{thm:general}. Part $(ii)$ of the theorem essentially implies that the error of the trained classifier will be close to $0.5$ (l.h.s) when $f \sim f^{CI}$ (under $\mathcal{H}_0$). On the other hand under $\mathcal{H}_1$ if $d_{TV}(f,f^{CI}) > 1 - \gamma$, the error will be less than $0.5(\gamma + b(n)) + \gamma_{n}$ which is small.

\section{Empirical Results}
In this section we provide empirical results comparing our proposed algorithm and other state of the art algorithms. The algorithms under comparison are: $(i)$ CCIT - Algorithm~\ref{alg:ccitv2} in our paper where we use XGBoost~\cite{chen2016xgboost} as the classifier. In our experiments, for each data-set we boot-strap the samples and run our algorithm $B$ times. The results are averaged over $B$ bootstrap runs\footnote{The python package for our implementation can be found \href{https://github.com/rajatsen91/CCIT}{here (https://github.com/rajatsen91/CCIT)}.}. $(ii)$ KCIT - Kernel CI test from~\cite{zhang2012kernel}. We use the Matlab code available online. $(iii)$ RCIT - Randomized CI Test from~\cite{strobl2017approximate}. We use the R package that is publicly available. 


\subsection{Synthetic Experiments}

We perform the synthetic experiments in the regime of \textit{post-nonlinear noise} similar to~\cite{zhang2012kernel}. In our experiments $X$ and $Y$ are dimension $1$, and the dimension of $Z$ scales (motivated by causal settings and also used in~\cite{zhang2012kernel,strobl2017approximate}). $X$ and $Y$ are generated according to the relation $G(F(Z)+\eta)$ where $\eta$ is a noise term and $G$ is a non-linear function, when the $\mathcal{H}_0$ holds. In our experiments, the data is generated as follows: $(i)$ when $X \independent Y \vert Z$, then each coordinate of $Z$ is a Gaussian with unit mean and variance, $X = \cos (a^TZ + \eta_1)$ and $Y = \cos (b^TZ + \eta_2)$. Here, $a, b \in \mathbb{R}^{d_z}$ and $\norm{a} = \norm{b} = 1$. $a$,$b$ are fixed while generating a single dataset. $\eta_1$ and $\eta_2$ are zero-mean Gaussian noise variables, which are independent of everything else. We set $Var(\eta_1) = Var(\eta_2) = 0.25$. $(ii)$ when $X \centernot{\independent} Y \vert Z$, then everything is identical to $(i)$ except that $Y = \cos (b^TZ + cX + \eta_2)$ for a randomly chosen constant $c \in [0,2]$.     


In Fig.~\ref{fig:dim}, we plot the performance of the algorithms when the dimension of $Z$ scales. For generating each point in the plot, $300$ data-sets were generated with the appropriate dimensions. Half of them are according to $\mathcal{H}_0$ and the other half are from $\mathcal{H}_1$ Then each of the algorithms are run on these data-sets, and the ROC AUC (Area Under the Receiver Operating Characteristic curve) score is calculated from the true labels (CI or not CI) for each data-set and the predicted scores. We observe that the accuracy of CCIT is close to $1$ for dimensions upto $70$, while all the other algorithms do not scale as well. In these experiments the number of bootstraps per data-set for CCIT was set to $B = 50$. We set the threshold in Algorithm~\ref{alg:ccitv2} to $\tau = 1/\sqrt{n}$, which is an upper-bound on the expected variance of the test-statistic when $\mathcal{H}_0$ holds.


\subsection{Flow-Cytometry Dataset}
\label{sec:real}
We use our CI testing algorithm to verify CI relations in the protein network data from the flow-cytometry dataset~\cite{sachs2005causal}, which gives expression levels of $11$ proteins under various experimental conditions. The ground truth causal graph is not known with absolute certainty in this data-set, however this dataset has been widely used in the causal structure learning literature. We take three popular learned causal structures that are recovered by causal discovery algorithms, and we verify CI relations assuming these graphs to be the ground truth. The three graph are: $(i)$ consensus graph from~\cite{sachs2005causal} (Fig. 1(a) in~\cite{mooij2013cyclic}) $(ii)$ reconstructed graph by Sachs et al.~\cite{sachs2005causal} (Fig. 1(b) in~\cite{mooij2013cyclic}) $(iii)$ reconstructed graph in~\cite{mooij2013cyclic} (Fig. 1(c) in~\cite{mooij2013cyclic}). 

For each graph we generate CI relations as follows: for each node $X$ in the graph, identify the set $Z$ consisting of its parents, children and parents of children in the causal graph. Conditioned on this set $Z$, $X$ is independent of every other node $Y$ in the graph (apart from the ones in $Z$). We use this to create all CI conditions of these types from each of the three graphs. In this process we generate over $60$ CI relations for each of the graphs. In order to evaluate false positives of our algorithms, we also need relations such that $X \centernot{\independent} Y \vert Z$. For, this we observe that if there is an edge between two nodes, they are never CI given any other conditioning set. For each graph we generate $50$ such non-CI relations, where an edge $X \leftrightarrow Y$ is selected at random and a conditioning set of size $3$ is randomly selected from the remaining nodes. We construct $50$ such negative examples for each graph.
\begin{figure*}
	\centering
	
	\subfloat[][]{\includegraphics[width = 0.31\linewidth]{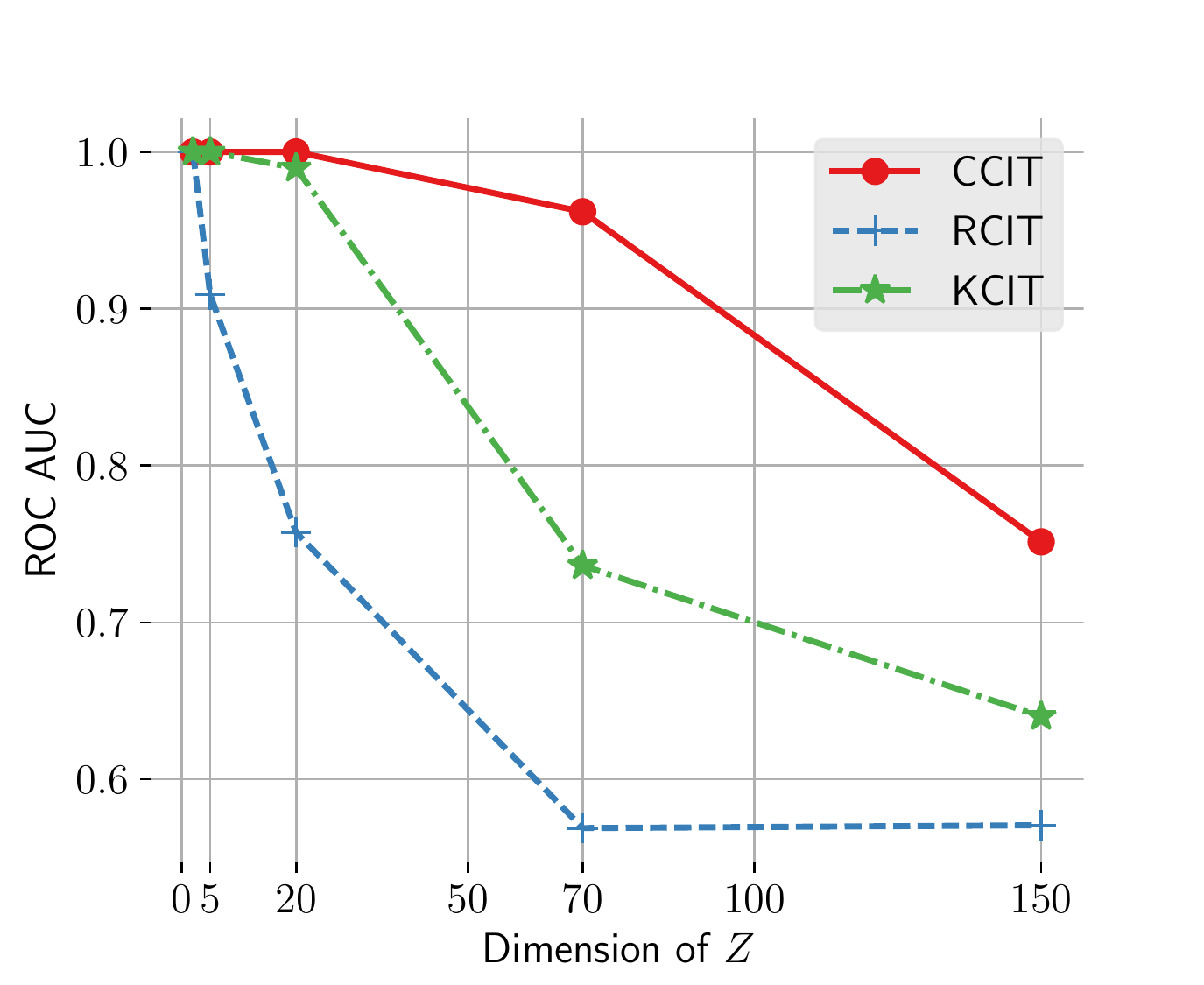}\label{fig:dim}} \hfill
	\subfloat[][]{\includegraphics[width = 0.31\linewidth]{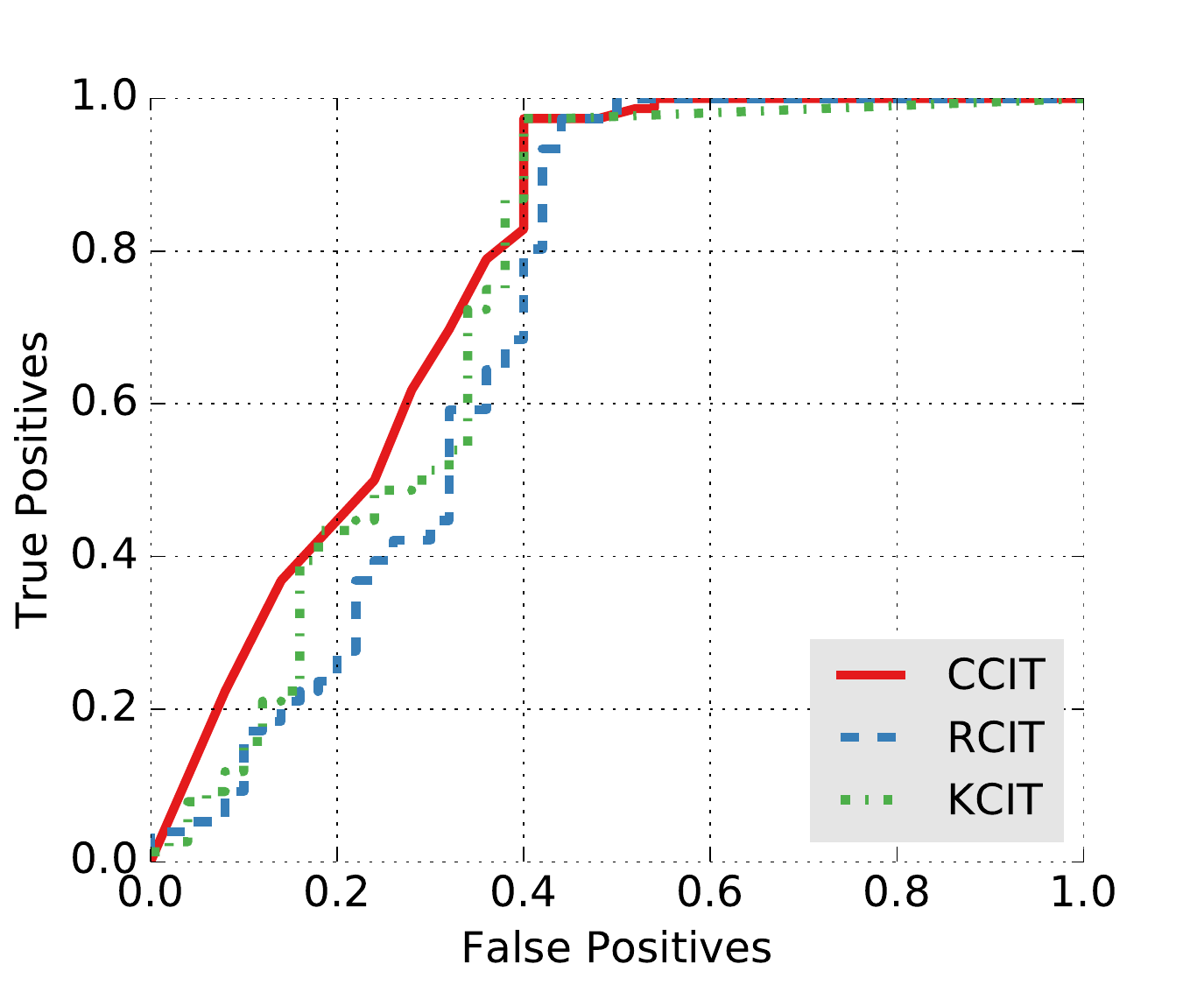}\label{fig:cyto}} \hfill
	\subfloat[][]{
		
		\resizebox{0.31\linewidth}{!}{
		\begin{tabular}[b]{ lccc }
			\hline
			Algo. & Graph $(i)$ & Graph $(ii)$ & Graph $(iii)$\\ \midrule 
			CCIT & {\color{blue} 0.6848} & {\color{blue} 0.7778} & {\color{blue} 0.7156}\\ 
			RCIT & 0.6448 & 0.7168 & 0.6928\\
			KCIT & 0.6528 & 0.7416 & 0.6610\\ \bottomrule
			\hline
		\end{tabular}
	}
		
		\label{fig:tab}

	} 
	\caption{ \small In (a) we plot the performance of CCIT, KCIT and RCIT in the post-nonlinear noise synthetic data. In generating each point in the plots, $300$ data-sets are generated where half of them are according to $\mathcal{H}_0$ while the rest are according to $\mathcal{H}_1$. The algorithms are run on each of them, and the ROC AUC score is plotted. In $(a)$ the number of samples $n = 1000$, while the dimension of $Z$ varies. In $(b)$ we plot the ROC curve for all three algorithms based on the data from Graph $(ii)$ for the flow-cytometry dataset. The ROC AUC score for each of the algorithms are provided in $(c)$, considering each of the three graphs as ground-truth. }
	\label{fig:cytoresults}
\end{figure*}
In Fig.~\ref{fig:cytoresults}, we display the performance of all three algorithms based on considering each of the three graphs as ground-truth. The algorithms are given access to observational data for verifying CI and non-CI relations. In Fig.~\ref{fig:cyto} we display the ROC plot for all three algorithms for the data-set generated by considering graph $(ii)$. In Table~\ref{fig:tab} we display the ROC AUC score for the algorithms for the three graphs. It can be seen that our algorithm outperforms the others in all three cases, even when the dimensionality of $Z$ is fairly low (less than 10 in all cases). An interesting thing to note is that the edges (pkc-raf), (pkc-mek) and (pka-p38) are there in all the three graphs. However, all three CI testers CCIT, KCIT and RCIT are fairly confident that these edges should be absent. These edges may be discrepancies in the ground-truth graphs and therefore the ROC AUC of the algorithms are lower than expected.

\FloatBarrier

\section{Conclusion}
In this paper we present a model-powered approach for CI tests by converting it into binary classification, thus empowering CI testing with powerful supervised learning tools like gradient boosted trees. We provide an efficient nearest-neighbor bootstrap which makes the reduction to classification possible. We provide theoretical guarantees on the bootstrapped samples, and also risk generalization bounds for our classification problem, under non-i.i.d near independent samples. In conclusion we believe that model-driven data dependent approaches can be extremely useful in general statistical testing and estimation problems as they enable us to use powerful supervised learning tools.

\bibliographystyle{plain}
\bibliography{citest.bib}

\clearpage

\appendix
\section{Guarantees on Bootstrapped Samples}
\label{proof:boot}
In this section we prove that the samples generated in Algorithm~\ref{alg:NNB}, through the nearest neighbor bootstrap, are close to samples generated from $f^{CI}(x,y,z) = f_{X \vert Z}(x \vert z)f_{Y \vert Z}(y \vert z)f_{Z}(z)$. The \textit{closeness} is characterized in terms of total variational distance as in Theorem~\ref{thm:dtv}. Suppose $2n$ i.i.d samples from distribution $f(x,y,z)$ are supplied to Algorithm~\ref{alg:NNB}. Consider a typical sample $(X,Y,Z) \sim f(x,y,z)$, which is modified to produce a typical sample in $\mathcal{U}_2'$ (refer to Algorithm~\ref{alg:NNB}) denoted by $(X,Y',Z)$. Here, $Y'$ are the $Y$-coordinates of a sample $(X',Y',Z')$ in $\mathcal{U}_2$ such that $Z'$ is the nearest neighbor of $Z$. Let us denote the marginal distribution of a typical sample in $\mathcal{U}_2'$ by $\phi_{X,Y,Z}(x,y,z)$, i.e $(X,Y',Z) \sim \phi_{X,Y,Z}(x,y,z)$. Now we are at a position to prove Theorem~\ref{thm:dtv}. 

\begin{proof}[Proof of Theorem~\ref{thm:dtv}]
Let $f_{Z'|z}(z')$ denote the conditional p.d.f of the variable $Z'$ (that is the nearest neighbor of sample $Z$ in $\mathcal{U}_2$),  conditioned on $Z=z$. Therefore, the distribution of the new-sample is given by,
\begin{align}
\phi_{X,Y,Z}(x,y,z) = f_{X \vert Z}(x \vert z)f_{Z}(z) \int f_{Y|Z}(y \vert z')f_{Z'|z}(z') dz' \label{eq:NNpdf}.
\end{align}
We want to bound the total variational distance between $\phi_{X,Y,Z}(x,y,z)$ and 
$f^{CI}_{X,Y,Z}(x,y,z)$. We have the following chain:

\begin{align*}
2*d_{TV}(\phi,f^{CI}) &= \int _{x,y,z} \Bigg \lvert f_{X \vert Z}(x \vert z)f_{Y \vert Z}(y \vert z)f_{Z}(z) - f_{X \vert Z}(x \vert z)f_{Z}(z) \int f_{Y|Z}(y \vert z')f_{Z'|z}(z') dz' \Bigg \rvert dxdydz \\
&=  \int _{x,y,z} f_{X \vert Z}(x \vert z)f_{Y \vert Z}(y \vert z)f_{Z}(z) \Bigg \lvert \int \left(1 - \frac{f_{Y|Z}(y \vert z')}{f_{Y|Z}(y \vert z)} \right) f_{Z'|z}(z') dz'\Bigg \rvert dxdydz \\
& \leq  \int _{x,y,z} f_{X \vert Z}(x \vert z)f_{Y \vert Z}(y \vert z)f_{Z}(z)  \int \Bigg \lvert 1 - \frac{f_{Y|Z}(y \vert z')}{f_{Y|Z}(y \vert z)} \Bigg \rvert  f_{Z'|z}(z') dz'dxdydz \\
& =  \int _{x,z,z'} f_{X \vert Z}(x \vert z)f_{Z}(z) f_{Z'|z}(z')  \left( \int \Bigg \lvert f_{Y|Z}(y \vert z)  - f_{Y|Z}(y \vert z') \Bigg \rvert  dy \right)   dz' dx dz \\
\end{align*}
\begin{align*}
& \leq  \int _{x,z,\lVert z'-z \rVert_2 \leq \epsilon} f_{X \vert Z}(x \vert z)f_{Z}(z) f_{Z'|z}(z')  \left( \int \Bigg \lvert f_{Y|Z}(y \vert z)  - f_{Y|Z}(y \vert z') \Bigg \rvert  dy \right)   dz' dx dz + \\
& 2 \int_{_{x,z,\lVert z'-z \rVert_2 > \epsilon}}  f_{X \vert Z}(x \vert z)f_{Z}(z) f_{Z'|z}(z') dz' dx dz \\
& \leq \int _{x,z,\lVert z'-z \rVert_2 \leq \epsilon} f_{X \vert Z}(x \vert z)f_{Z}(z) f_{Z'|z}(z')  \left( \int \Bigg \lvert f_{Y|Z}(y \vert z)  - f_{Y|Z}(y \vert z') \Bigg \rvert  dy \right)   dz' dx dz + \\
& 2* \PP\left( \lVert z'-z \rVert_2 > \epsilon \right) \numberthis    
\label{eq:intdiff}
\end{align*}

By Pinsker's inequality, we have:
\begin{align}
\int _{y} \Bigg \lvert f_{Y|Z}(y \vert z) -f_{Y|Z}(y \vert z')  \Bigg \rvert dy \leq \sqrt{ \frac{1}{2} \int_y  \log \frac{f(y|z)}{f(y|z')}  f(y|z) dy }
\label{eqn:pinsker}
\end{align}


By Taylor's expansion with second-order residual, we have:
\begin{align}
\int \log \frac{f(y|z)}{f(y|z')}  f(y|z) dy  =  \frac{1}{2}(z'-z)^T \mathbf{I}_a(z) (z'-z)
\label{eqn:taylor}
\end{align}
for some $a = \lambda z + (1-\lambda)z'$ where $0 \leq \lambda \leq 1$. 

Under Assumption~\ref{assump1} and $\epsilon < \epsilon_1$ we have,
\begin{align}
(z'-z)^T \mathbf{I}_a(z) (z'-z) \leq \beta \lVert z' -z  \rVert_2^2.
\label{eqn:boundeigen}
\end{align}

Then, (\ref{eqn:boundeigen}), (\ref{eqn:taylor}), (\ref{eqn:pinsker}) and (\ref{eq:intdiff}) imply:
\begin{align}
2*d_{TV}( \phi,f^{CI} )  \leq \sqrt{\frac{\beta}{4} \mathbb{E}[ \lVert z'-z \rVert_2 \mathds{1}_{ \lVert z'-z \rVert_2 \leq \epsilon} ] } + 2 \PP\left( \lVert z'-z \rVert_2 > \epsilon \right) 
\label{eqn:expdtv}
\end{align}

We now bound both terms separately. Let $Z_1,Z_2 ...,Z_n$ be distributed i.i.d according to $f(z)$. Then, $f_{Z'|z} (\cdot)$ is the pdf of the nearest neighbor of $z$ among $Z_1, \ldots, Z_n$. 
\subsubsection{Bounding  the first term}

%

In this section we will use $d$ in place of $d_z$ for notational simplicity. Let $\gamma_d$  be the volume of the unit $\ell_2$ ball in dimension $d$. Let $S=\{z: f(z) \geq 2*c_d \epsilon^2 \}$. This implies, that for $z \in S$:
\begin{align} \label{eqn:ineq1}
f(z) - c_d \epsilon^2 \geq f(z)/2
\end{align}

Let $Z' = \argmin_{Z_1,Z_2 \ldots, Z_{n}} \lVert  Z_i - z \rVert_2$ be the random variable which is the nearest neighbor to a point $z$ among $n$ i.i.d samples $Z_i$ drawn from the distribution whose pdf is $f(z)$ that satisfies assumption \ref{assumption1}. Let $r(z) = ||z-z'||_2 $. Let $F(r)$ be the CDF of the random variable $R$. Since $R$ is a non-negative random variable,
\begin{align} \label{eqn:firsttermbound}
\mathbb{E}_{R}[ r(z) \mathds{1}_{r \leq \epsilon}] = \int_{r \leq \epsilon} r dF(r) = \left[ rF(r) \right]_{0}^{\epsilon} - \int_{r \leq \epsilon} F(r) dr \leq \int _{r \leq \epsilon} P(R>r) dr 
\end{align}
For any $r \leq \epsilon$, observe that
\begin{align*}
\Pr(R > r)
& = \text{Pr}( \nexists i: z_i \in B(z,r)) \\
& = \left(1 - \text{Pr}( Z \in B(z,r))\right)^{n} \\
& \leq \exp ( - n \text{Pr}( Z \in B(z,r))) \numberthis \label{eqn:boundprob1}
\end{align*}

We have the following chain to bound $\text{Pr}( Z \in B(z,r))$. Let $a=\lambda z + (1-\lambda) t$.
\begin{align*}
\lvert \text{Pr}( Z \in B(z,r)) - f(z) \gamma_d r^d \rvert  &\leq  \Bigg \lvert \int_{t \in B(z,r)}  (f(t)-f(z)) dt  \Bigg \rvert \\
\hfill   &= \Bigg \lvert \int_{t \in B(z,r)} (\nabla f(z)^T .(t-z) + (t-z)^T H_f(a) (t-z))dt \Bigg \rvert \\
\hfill & \leq \max_{t \in B(z,r)} \lVert H_f(t) \rVert_2  \int_{t \in B(z,r)} \lVert t-z \rVert_2^2 dt  \\
& \leq c_d \gamma_d r^{d+2} \numberthis \label{eqn:boundprob2}
\end{align*}
By putting together (\ref{eqn:firsttermbound}),(\ref{eqn:boundprob1}) and(\ref{eqn:boundprob2}), we have:
\begin{align*}
\mathbb{E}_{R}[ r(z) \mathds{1}_{r \leq \epsilon}] &{\leq} \mathds{1}_{z \in S} \left( \int_{r \leq \epsilon} e^{- n  \gamma_d r^d (f(z)- c_d r^2 )} dr \right)  + \epsilon \mathds{1}_{z \in S^c} \\
\hfill        & \leq \mathds{1}_{z \in S}  \left( \int_{r \leq \epsilon} e^{- \frac{1}{2}n \gamma_d f(z) r^d } dr \right)+ \epsilon \mathds{1}_{z \in S^c} \\
\hfill        & \leq  \mathds{1}_{z \in S}  \left( \int_{t \leq \frac{1}{2} n \gamma_d f(z) \epsilon^d} \frac{e^{-t} }{d [\frac{1}{2} n \gamma _d f(z)]^{1/d} } t^{-1+1/d} dt  \right) +  \epsilon \mathds{1}_{z \in S^c} \\
\hfill        & \leq  \mathds{1}_{z \in S} \left( \frac{2^{1/d}}{d (n \gamma_d f(z))^{1/d}} \Gamma(1/d) \right)+ \epsilon \mathds{1}_{z \in S^c}
\end{align*}

Therefore, the first term in bounded by:
\begin{align}
\mathbb{E}[ \lVert z'-z \rVert_2 \mathds{1}_{ \lVert z'-z \rVert_2 \leq \epsilon} ] & \leq \mathbb{E}_Z \left[ \mathbb{E}_{R}[ r(z) \mathds{1}_{r \leq \epsilon}] \right] \\
&  \leq  \mathbb{E}_Z \left[ \frac{2^{1/d}}{d (n \gamma_d f(z))^{1/d}} \Gamma(1/d) + \epsilon \mathds{1}_{z \in S^c} \right] \\
\hfill & \leq \frac{c_3 * 2^{1/d}}{(n \gamma_d)^{1/d} d} \Gamma(1/d) + \epsilon * G \left(2c_d \epsilon^2 \right) \\
\end{align}

\subsubsection{ Bounding the second term}
We now bound the second term as follows:
\begin{align*}
\Pr(||z - z'||_2 > \epsilon) & \leq \mathbb{E}_Z [ \text{Pr} (R>\epsilon) ]  \\
& \leq \mathbb{E}_Z  \left[ \mathds{1}_{z \in S} \exp \left( - \frac{n \gamma_df(z) \epsilon^d }{2} \right) \right] +  \mathrm{Pr} ( z \in S^c)  \\
& \leq \exp \left( -n \gamma_d c_d \epsilon^{d+2} /2\right) + G \left( 2 c_d \epsilon^2 \right)
\end{align*}
Substituting in (\ref{eqn:expdtv}), we have:
\[ 2*d_{TV}( g,f^{CI} )  \leq   \sqrt{\frac{\beta}{4} \frac{c_3 * 2^{1/d} \Gamma(1/d) }{(n \gamma_d)^{1/d} d}  + \frac{\beta \epsilon G  \left(2c_d \epsilon^2 \right) }{4} }+ 2 \exp \left( -n \gamma_d c_d \epsilon^{d+2}/2 \right) + 2G \left( 2 c_d \epsilon^2 \right)
\]

Substitute $d_z$ in place of $d$ to recover Theorem~\ref{alg:NNB}. 

\end{proof}  

\section{Generalization Error Bounds on Classification}
\newcommand{\cU}{\mathcal{U}}
\newcommand{\cZ}{\mathcal{Z}}
\newcommand{\cD}{\mathcal{D}}
\newcommand{\cO}{\mathcal{O}}
\newcommand{\E}{\mathbb{E}}

In this section, we will prove generalization error bounds for our classification problem in Algorithm~\ref{alg:ccitv1}. Note the samples in $\mathcal{U}_2'$ are not i.i.d, so standard risk bounds do not hold. We will leverage a spatial near independence property to provide generalization bounds under non-i.i.d samples. In what follows, we will prove the results for any bounded loss function $L(g(u), \ell) \leq \lvert L \rvert.$ Let $S \triangleq \cD_r$ i.e., the set of training samples. For
$1 \leq i \leq 3$, let $\cZ_i \triangleq \{z :
(x,y,z) \in \cU_i\}$. Let $\cZ \triangleq \cZ_1 \cup \cZ_2$.  Observe
that
\begin{equation}
\label{eq:usual}
R_{{q}}({g}_S) \leq R_{{q}} ({g}^*_{q}) + 2\sup_{g \in \cG} (R_S(g) - R_{{q}}(g)),
\end{equation}
and hence in the rest of the section we upper bound $\sup_{g \in \cG}
(R_S(g) - R_{{q}}(g))$.  To this end, we define conditional risk
$R_S(q | \cZ)$ as
\[
R_S(g |\cZ) \triangleq  \frac{1}{n} \sum_{(u,\ell) \in S} \EE[L(g(\xyz),\ell)|\cZ].
\]
By triangle inequality,
\begin{equation}
\label{eq:dissociate}
\sup_{g \in \cG} (R_S(g) - R_{{q}}(g)) \leq \sup_{g \in \cG} (R_S(g) - R_{{S}}(g | \cZ)) + \sup_{g \in \cG} (R_S(g|\cZ) - R_{{q}}(g)).
\end{equation}
We first bound the second term in the right hand side of Equation~\eqref{eq:dissociate} in the next lemma.
\begin{lemma}
\label{lem:cond2}
With probability at least $1-\delta$, 
\[
\sup_{g \in \cG} (R_S(g|\cZ) - R_{{q}}(g)) \leq |L|C  \sqrt{\frac{V}{n}} + |L| \sqrt{\frac{2 \log (1/\delta)}{n}},
\]
where $V$ is the VC dimension of the classification model.
\end{lemma}
\begin{proof}
For a sample $u= (x,y,z)$, observe that $\cZ \to z \to
u$ forms a Markov chain. Hence,
\[
  \E[L(g(u),\ell)|\cZ] = \E[L(g(u),\ell)|Z].
\]
Let
\begin{equation}
\label{eq:hdef}
h(Z) \triangleq  \E[L(g(u),\ell)|Z]
\end{equation}
Hence,
\[
R_S(g|\cZ) - R_{{q}}(g) = \frac{1}{n}\sum_{z \in \cZ} h(Z) - \E_q[h(Z)].
\]
The above term is the average of $n$ independent random variables $h(Z)$ and
hence we can apply standard tools from learning theory~\cite{boucheron2005theory} to obtain
\[
\sup_{g \in \cG} (R_S(g|\cZ) - R_{{q}}(g)) \leq |L|C \sqrt{\frac{V_h}{n}}
+ |L|\sqrt{\frac{2 \log (1/\delta)}{n}},
\]
where $V_h$ is the VC dimension of the class of models of $h$. The
lemma follows from the fact that VC dimension of $h$ is smaller than
the VC dimension of the underlying classification model.
\end{proof}
We next bound the first term in the RHS of
Equation~\eqref{eq:dissociate}. Proof is given in
Appendix~\ref{app:cond3}.
\begin{lemma}
\label{lem:cond3}
Let $\epsilon > 0$. If the Hessian of the density $f(z) $and the Lipscitz constant of the same is bounded, then with probability at least
  $1-7\delta$, \begin{align*}
  \sup_{g \in \cG}
  (R_S(g) - R_{{n}}(g | \cZ)) &\leq |L| \left(\sqrt{V}
  + \sqrt{\log \frac{1}{\delta}} \right)\left( \left(\frac{\log (n/\delta)}{n}\right)^{1/3} +  \sqrt{\frac{4^d\log (n/\delta) + o_n(1/\epsilon)}{n}}\right)  \\ 
  &+ 
  |L| G(\epsilon). \numberthis  \label{eq:second_term}   \end{align*}
\end{lemma}
Lemmas~\ref{lem:cond2} and~\ref{lem:cond3}, together with Equations~\eqref{eq:usual}
and~\eqref{eq:dissociate} yield the following theorem.
\begin{theorem}
\label{thm:general}
Let $\epsilon > 0$.  If the Hessian and the Lipscitz constant of $f(z)$ is bounded, then with probability at least
  $1-8\delta$, 
\begin{equation}
\label{eq:general}
R_{{q}}(\hat{g}) \leq R_{{q}} ({g}^*_{q}) + c|L| \left( \left(\sqrt{V}
  + \sqrt{\log \frac{1}{\delta}} \right)\left( \left(\frac{\log
  (n/\delta)}{n}\right)^{1/3} + \sqrt{\frac{4^d\log (n/\delta)  + o_n(1/\epsilon)}{n}}\right)
  + G(\epsilon) \right),
\end{equation}
where $c$ is a universal constant and $\hat{g}$ is the minimizer in
Step $6$ of Algorithm~\ref{alg:ccitv1}.
\end{theorem}
\subsection{Proof of Lemma~\ref{lem:cond3}}
\label{app:cond3}
We need few definitions to prove Lemma~\ref{lem:cond3}.
For a point $z$, let $B_n(z)$ be a ball around it such that
\[
\Pr_{Z \sim f(z)} (Z \in B_n(z)) = \frac{\log \frac{n^2}{\delta}}{n} \triangleq \alpha_n.
\]
Intuitively, with high probability the nearest neighbor of each sample
$z$ lies in $B_n(z)$. We formalize it in the next lemma.
\begin{lemma}
  \label{lem:near_tech} With probability $>1 - \delta$, the nearest
neighbor of each sample $z \in \cZ_2$ in $\cZ_3$ lie in $B(z)$.
\end{lemma}
\begin{proof}
The probability that none of $\cZ_3$ appears in $B(z)$ is
  $1 - (1-\alpha_n)^{n} \leq \delta/n$. The lemma follows by the union
  bound.
  \end{proof}
We now bound the probability that the the
  nearest neighbor balls $B_n()$ intersect for two samples.
\begin{lemma}
  \label{lem:beta_bound} Let $\epsilon > 0$. If the Hessian of the
  density ($f(z)$) is bounded by $c$ and the Lipschitz constant is bounded by $\gamma$, then
  for any given $z_1$ such that $f(z_1) \geq \epsilon$ and a sample
  $z_2 \sim f$, \[ \Pr_{z_2 \sim f} (B_n(z_1) \cap
  B_n(z_2) \neq \emptyset) \leq \beta_n \triangleq 4^{d} \alpha_n
  (1+o_n(1/\epsilon)).
\]
\end{lemma}
\begin{proof}
  Let $r_n(z)$ denote the radius of $B_n(z)$. Let $B(z,r)$ denote the
  ball of radius $r$ around $z$ and $V(z,r)$ be its volume.  We can
  rewrite $\beta_n$ as
  \begin{align*}
   & =  \Pr(B_n(z_1) \cap B_n(z_2) \neq \emptyset) \\
    & = \Pr(B_n(z_1) \cap B_n(z_2) \neq \emptyset, 3 r_n(z_1) \geq r_n(z_2))
    + \Pr(B_n(z_1) \cap B_n(z_2) \neq \emptyset,  3r_n(z_1) < r_n(z_2)).
  \end{align*}
  We first bound the first term. Note that
  \[
\int_{z' \in B_n(z)} f(z') dz' = \alpha_n,
  \]
  Hence, by Taylor's series expansion and the bound on Hessian yields,
  \[
\alpha_n = \int_{z' \in B_n(z)} f(z') dz' = V(z, r_n(z)) \left(f(z) + O( r^2_n(z) c)\right) .
\]
Similarly,
\begin{align*}
\Pr(z' \in V(z,4r_n(z)))
&= V(z, 4r_n(z)) \left(f(z) +O(9 r^2_n(z) c)\right) \\
& = 4^d  \alpha_n (1 + o_n(1/\epsilon)),
\end{align*}
where the last equality follows from the fact that $V(z, 4r_n(z))/
  V(z, r_n(z)) = 4^d$ in $d$ dimensions.

Then the first term can be
  bounded as
  \begin{align*} \Pr(B_n(z_1) \cap B_n(z_2) \neq \emptyset,
  3r_n(z_1) \geq r_n(z_2)) & = \Pr(z_2 \in B(z_1, r_n(z_1) +
  r_n(z_2)), 3 r_n(z_1) \geq r_n(z_2)) \\ & \leq \Pr(z_2 \in B(z_1,
  4r_n(z_1)), 3r_n(z_1) \geq r_n(z_2)) \\ & \leq \Pr(z_2 \in B(z_1,
  4r_n(z_1))) \\ & \leq 4^{d} \alpha_n (1 +
  o_n(1/\epsilon)).
  \end{align*}
  To bound the second term, observe
  that if $B_n(z_1) \cap B_n(z_2) \neq \emptyset$ and $3r_n(z_1) <
  r_n(z_2)$. There exists a point $z'$ on the line joining $z_1$ and
  $z_2$ at distance $3r_n(z_1)$ from $z_1$ such that \[ \Pr(z'' \in
  B(z', 3r_n(z_1))) < \alpha_n.  \] As before bound on the Hessian
  yields, \[
\alpha_n > V(z', 3r_n(z_1)) (f(z') - O(9 r^2_n(z_1) c)).
\]
Hence,
\[
f(z') < 3^{-d}( f(z) + O(r^2_n(z) c)) + O(9 r^2_n(z_1) c).
\]
However, $f(z) > \epsilon$ and $f(z') \geq f(z) - 3r_n(z_1)\gamma$ and
$r_n(z_1) \to 0$. Hence, a contradiction. Thus
\[
\Pr(B_n(z_1) \cap B_n(z_2) \neq \emptyset,  3r_n(z_1) < r_n(z_2)) = 0.
\]
\end{proof}
Consider the graph on indices $[n]$, such that two indices are
connected if and only if $B_n(z_i) \cap B_n(z_j) \neq \emptyset$,
$f(z_1) \geq
\epsilon$, $f(z_2) \geq \epsilon$ . Let $\Delta(Z^n_1)$ be the maximum degree of the resulting graph. We first show that the maximum
degree of this graph is small.
\begin{lemma}
  With probability $\geq 1-\delta$,
  \[
  \Delta(Z^n_1) \leq 4n \beta_n.
  \]
  \end{lemma}
\begin{proof}
  For index $1$, by Lemma~\ref{lem:beta_bound} that probability of $j$
  points intersect is at most
  \[
\sum^j_{i=0} {n \choose i } \beta^i_n (1-\beta_n)^{n-i}.
 \] Hence, the degree of vertex $1$ is dominated by a binomial
 distribution with parameters $n$ and $\beta_n$. The lemma follows
 from the union bound and the Chernoff bound.  \end{proof} Let $k >
 2\Delta(Z^n_1)$ and $S_0, S_1,S_2\ldots S_k$ be $k$ independent sets
 of the above graph such that $\max_{t\geq 1}|S_t| \leq 2n/k$. Note
 that such independent sets exists by Lemma~\ref{lem:indi_size}. We
 set the exact value of $k$ later. Let $S_0$ contains all indices such
 that $f(z_i) < \epsilon$.
\begin{lemma}
  With probability $> 1-\delta$,
  \[
   |S_0| \leq n G(\epsilon) + \sqrt{n \log \frac{1}{\delta}}.
  \]
  \end{lemma}
\begin{proof}
Observe that $|S_0|$ is the sum of $n$ independent random variables
and changing any of them changes $S_0$ by at most $1$. The lemma
follows by McDiarmid's inequality.
\end{proof}
We can upper bound the LHS in Equation~\eqref{eq:second_term} as
\[
\sup_{g \in \cG} (R_S(g) - R_{{n}}(g | \cZ)) \leq \sum^k_{t=1}
\frac{|S_t|}{n} \sup_{g \in \cG} \frac{1}{|S_t|} \sum_{i \in S_t}
\left(L(g(\xyz_i),\ell_i) - h(Z_i)\right) + \frac{|S_0|}{n} |L|.
\]
Let $N(Z_i)$ denote the number of elements of $\cZ_3$ that are in
$B(Z_i)$ and Let $A_i$ be always true if $Z_i \in \cZ_1$ and otherwise
$A_i$ be the event such that nearest neighbor of samples in $N(Z_i) >
0$.  We first show the following inequality.
\begin{lemma}
With probability $\geq 1-\delta$, for all sets $S_t$.
  \[
 \sup_{g \in \cG} \frac{1}{|S_t|}\sum_{i \in S_t}  \left(L(g(\xyz_i),\ell_i ) -
 h(Z_i)\right) =  \sup_{g \in \cG} \frac{1}{|S_t|}\sum_{i \in S_t}  \indic_{A_i}\left(L(g(\xyz_i),\ell_i) -
 h(Z_i)\right).
  \]
\end{lemma}
\begin{proof}
  Let $X_i =  \left(L(g(\xyz_i),\ell_i) -
  h(Z_i)\right)$ and
  Observe that LHS can be written as 
  \begin{align*}
\sup_{g \in \cG} \frac{1}{|S_t|} \sum_{i \in S_t} X_i = \sup_{g \in \cG} \frac{1}{|S_t|} \sum_{i \in S_t} X_i \indic_{A_i} + \sup_{g \in \cG} \frac{1}{|S_t|} \sum_{i \in S_t}  (X_i - X_i \indic_{A_i}).
  \end{align*} If the conditions of Lemma~\ref{lem:near_tech} hold,
  the nearest sample of $Z_i$'s lie within $B_n(Z_i)$. Hence, with
  probability $\geq 1- \delta$, the second term is $0$. Hence the
  lemma.
\end{proof}
Let $r$ be defined as follows.
\[
r(Z_i, N_i) \triangleq \EE[  \indic_{A_i}\left(L(g(\xyz_i),\ell_i) -
 h(Z_i)\right) | N(Z_i) = N_i].
\]
Observe that
\begin{align*}
& \EE\left[\sum_{i \in S_t}  \indic_{A_i}\left(L(g(\xyz_i),\ell_i ) -
  h(Z_i)\right) | N(Z_1),\ldots N(Z_n) \right] \\
& = \sum_{i \in S_t}\EE\left[  \indic_{A_i}\left(L(g(\xyz_i),\ell_i ) -
  h(Z_i)\right) | N(Z_1),\ldots N(Z_n) \right] \\
& = \sum_{i \in S_t}
\EE\left[  \indic_{A_i}\left(L(g(\xyz_i),\ell_i ) -
  h(Z_i)\right) | N(Z_i) \right] \\ \numberthis \label{eq:imp_step}
& = \sum_{i \in S_t}  r(Z_i, N_i).
\end{align*}
Hence, we can split the term as
\begin{align*}
& \sup_{g \in \cG} \frac{1}{|S_t|}\sum_{i \in S_t}  \indic_{A_i}\left(L(g(\xyz_i),\ell_i ) -
 h(Z_i)\right) \\
 &\leq
  \sup_{g \in \cG} \frac{1}{|S_t|}\sum_{i \in S_t}  \indic_{A_i}\left(L(g(\xyz_i),\ell_i ) -
 r(Z_i, N(Z_i))\right) +  \sup_{g \in \cG} \frac{1}{|S_t|}\sum_{i \in S_t}  \indic_{A_i}\left( r(Z_i, N(Z_i))-
 h(Z_i)\right) \numberthis \label{eq:two_again}
\end{align*}

Given $\{Z_i, N(Z_i) \}$, the first term in the RHS of the Equation~\eqref{eq:two_again}  is a function of $|S_t|$ independent random variables as  $\indic_{A_i}* L(g(\xyz_i),\ell_i) $ are mutually
 independent given $\{Z_i, N(Z_i) \}$.
\ignore{
This is because of the following lemma:
\begin{lemma}
The variables $\indic_{A_i}* L(g(\xyz_i),\ell_i) $ are mutually
 independent given $\{Z_i, N(Z_i) \}$.  \end{lemma} \begin{proof} Let
 $j^{(i)}_1,j^{(i)}_2 ....j^{(i)}_{N(Z_i)}$ be the indices of $N(Z_i)$
 points ,sorted according to increasing order of indices, that fall
 inside $B(Z_i)$ whenever $N(Z_i)>0$. Then, for $i$ such that
 $N(Z_i)>0$. Then, the joint distribution of the variables given
 $\{Z_i, N(Z_i)\}$ $\{ j^{(i)}_1,j^{(i)}_2 ....j^{(i)}_{N(Z_i)},
 Z_{j^{(i)}_1}, \ldots Z_{j^{(i)}_{N(Z_i)}} \}$ is distributed
 according to the probability density function $f_{\{Z_i,
 N(Z_i)\}}^{i}( \mathbf{n}, \mathbf{w} )$ (note that some co-ordinates
 (i.e, $\mathbf{n}$) are discrete while the rest of them (i.e,
 $\mathbf{w}$) are continuous). Observe that $f^{i}
 ( \cdot, \mathbf{w} ) \neq 0 $ only if $w_k \in B(z_i),~ \forall w_k
 $. Consider the partially marginalized distribution
 $g^{i}(\mathbf{w})= \sum \limits_{\mathbf{n}}f^{i}(\mathbf{n},\mathbf{w})$. Consider
 $j \neq i$ such that $N(Z_i)>0$. It is easy to see that the
 coordinates for $g^{i}$ and $g^{j}$ have disjoint support. (NOT CLEAR
 HERE) Consider the joint distribution $g^{i,j}$ whose marginals are
 $g^{(i)}$ and $g^{(j)}$. Because the supports are disjoint, it must
 be that $g^{(i,j)}= g^{(i)} g^{(j)}$. Note that the distribution of
 $a_i$ is a function of the $N_{Z_i}$ random variables $\{
 Z_{j^{(i)}_1}, \ldots Z_{j^{(i)}_{N(Z_i)}} \}$. Therefore, $\{a_i:
 N(Z_i)>0\}$ are mutually independent.  \end{proof}}
Thus we can use
 standard tools from VC dimension theory and state that with
 probability $\geq 1-\delta$, the first term in the RHS of Equation~\eqref{eq:two_again} can be upper bounded as 
\[
\sup_{g \in \cG} \frac{1}{|S_t|}\sum_{i \in S_t}  \indic_{A_i}\left(L(g(\xyz_i),\ell_i ) -
 r(Z_i, N(Z_i))\right)  \leq |L| C \sqrt{\frac{V}{|S_t|}} + |L|\sqrt{\frac{\log(1/\delta)}{|S_t|}}.
\]
conditioned on $\{Z_i, N(Z_i) \}$

To bound the second term in the RHS of Equation~\eqref{eq:two_again},
observe that unlike the first term, the $N(Z_i)$s are dependent on
each other. However note that $N(Z_1),\ldots N(Z_{|S_t|})$ are
distributed according to multinomial distribution with parameters $n$
and $\alpha_n$. However, if we replace them by independent Poisson
distributed $N(Z_i)$s we expect the value not to change.  We formalize
it by total variation distance.  By Lemma~\ref{lem:roos}, the total
variation distance between a multinomial distribution and product of
Poisson distributions is 
\[
\cO( |S_t| \alpha_n),
\]
and hence any bound holds in the second distribution holds in the first one with an additional penalty of
\[
\cO \left( |S_t| \alpha_n |L|\right).
\]
Under the new independent sampling distribution, again the samples are
independent and we can use standard tools from VC dimension and hence,
with probability $\geq 1-\delta$, the term is upper bounded by
\[
|L|C \sqrt{\frac{V}{|S_t|}} + |L|\sqrt{\frac{\log(1/\delta)}{|S_t|}}.
\]
Hence, summing over all the bounds, we get 
\begin{align*}
\sup_{g \in \cG}
  (R_S(g) - R_{{n}}(g | \cZ)) & \leq 
|L|\cO \left( \frac{ |S_0|}{n} + \sum^k_{t=1} \frac{|S_t|}{n} \left( \sqrt{\frac{\log
    \frac{1}{\delta}}{|S_t|}} + c\sqrt{\frac{V}{|S_t|}} + \alpha_n
|S_t| \right) \right) \\
& \leq |L| \cO \left( G(\epsilon) +  \sqrt{\frac{k}{n}} \left(\sqrt{ \log
  \frac{1}{\delta}} + c\sqrt{V}\right) + \alpha_n \max_t |S_t| + \sqrt{\frac{\log \frac{1}{\delta}}{n}} \right)\\
& \leq
|L| \cO \left( \sqrt{\frac{k}{n}} \left(\sqrt{ \log \frac{1}{\delta}} +
c\sqrt{V}\right) + \alpha_n \frac{n}{k}+ \sqrt{\frac{\log \frac{1}{\delta}}{n}} \right).
\end{align*}
conditioned on $Z_i, N(Z_i)$
Choose $k =  n \alpha^{2/3}_n+ 8n\beta_n$, and note that the conditioning can be removed as the term on the r.h.s are constants.This yields the result. The error
probability follows by the union bound.

\qedhere

\begin{theorem}
	\label{thm:riskdtv}
	Assume the conditions for Theorem~\ref{thm:general}. Suppose the loss is $L(g(u),\ell) = \mathds{1}_{g(u) \neq \ell}$ (s.t $|L| \leq 1$). Further suppose the class of classifying function is such that $R_{q}(g^*_q) \leq r_{0} + \eta$. Here, 
	$r_0 \triangleq 0.5(1 - d_{TV}(q(x,y,z \vert 1), q(x,y,z \vert 0)))$ is the risk of the Bayes optimal classifier when 
	$\PP(\ell = 1) = \PP(\ell=0)$. This is the best loss that any classifier can achieve for this classification problem~\cite{boucheron2005theory}. Under this setting, w.p at least $1 - 8\delta$ we have:
	\begin{align*}
	\frac{1}{2} \left(1 - d_{TV}(f,f^{CI}) \right) - \frac{b(n)}{2}  \leq R_{{q}}(g_{S}) \leq  \frac{1}{2} \left(1 - d_{TV}(f,f^{CI}) \right) + \frac{b(n)}{2} + \eta + \gamma_n
	\end{align*}
	
\end{theorem}

\begin{proof}
	Assume the bounds of Theorem~\ref{thm:general} holds which happens w.p at least $1-8\delta$. From Theorem~\ref{thm:general} we have that \begin{equation} R_{{q}}(g_{S}) \leq R_{{q}}(g_{q}^*) + \gamma_n \label{eq:g}. \end{equation} Also, note that from Theorem~\ref{thm:dtv} we have the following:
	\begin{align*}
	d_{TV}(q(x,y,z \vert 1), q(x,y,z \vert 0)) &= d_{TV}(\phi,f) \\
	&\leq d_{TV}(\phi,f^{CI}) + d_{TV}(f^{CI},f) \\
	&\leq b(n) + d_{TV}(f^{CI},f) \numberthis \label{eq:triangle}
	\end{align*}
	Under our assumption we have $R_{{q}}(g_{q}^*) \leq r_0 + \eta$. Combining this with \eqref{eq:g} and \eqref{eq:triangle} we get the r.h.s. For, the l.hs note that  $R_{q}(g^*_q) \geq r_0$ as the bayes optimal classifier has the lowest risk. We can now use~\eqref{eq:triangle} to prove the l.h.s. 
\end{proof}

\section{Tools from probability and graph theory}

\begin{lemma}[McDiarmid's inequality~\cite{mcdiarmid1989method}]
Let $X_1,X_2,\ldots X_m$ be $m$ independent random variables and $f$
be a function from $x^n_1 \to \mathbb{R}$ such that changing any one
of the $X_i$s changes the function $f$ at most by $c_i$, then
\[
\Pr(f - \E[f] \geq \epsilon) \leq \exp \left(\frac{-2\epsilon^2}{\sum^m_{i=1} c^2_i}\right).
\]
\end{lemma}
\begin{lemma}[Special case of Theorem $1$ in ~\cite{Roos99}]
\label{lem:roos}
Let $f_m$ be the multinomial distribution with parameters $n$ and
$p_1$, $p_2$, \ldots $p_k, 1- \sum^k_{i=1} p_i$, and $f_p$ be the
product of Poisson distributions with mean $np_i$ for $i \leq 1\leq
k$, then
\[
d_{TV}(f_m, f_s) \leq 8.8 \sum^k_{i=1} p_i.
\]
\end{lemma}

\begin{lemma}
\label{lem:indi_size}
For a graph with maximum degree $\Delta$, there exists a set of independent sets $S_1,S_2,\ldots S_k$
such that $k \geq 2\Delta$ and
\[
\max_{1\leq i \leq k} |S_i| \leq 2n/k.
\]
\end{lemma}
\begin{proof}
We show that the following algorithm yields a coloring (and hence
independent sets) with the required property.

Let $1,2,\ldots k$ be $k$ colors, where $k > 2\Delta$. We arbitrarily
order the nodes, and sequentially color nodes with a currently least
used color from among the ones not used by its neighbors.
Consider the point in time when $i$ nodes have been colored, and we
evaluate the options for the $(i+1)^{th}$ node. The number of possible
choices of color for that node is $c \geq k - \Delta$. Out these $c$
colors, the average number of nodes belonging to each color at this
point is at-most $i/c$. Therefore by pigeonholing, the minimum is less
than the average; thus the number of nodes belonging to chosen color
is no larger than $i/c \leq i/(k - \Delta)$.

Hence at the end when all $n$ nodes are colored, each color has been
used no more than $(n-1)/(k - \Delta) + 1 < 2n/k$.
\end{proof}


\end{document}